\pdfoutput=1

\documentclass{article}
\usepackage{ijcai20}

\usepackage{times}
\usepackage{soul}
\usepackage{url}
\usepackage[hidelinks]{hyperref}
\usepackage[utf8]{inputenc}
\usepackage[small]{caption}
\usepackage{graphicx, xcolor}
\usepackage{amsmath, amsthm, amssymb}
\usepackage{mathtools}
\usepackage{booktabs}
\usepackage{algorithm, algorithmic}
\usepackage[capitalize]{cleveref}
\usepackage{bm}
\usepackage{dsfont}
\usepackage{float}
\allowdisplaybreaks

\title{Individual Fairness Revisited:\\ Transferring Techniques from Adversarial Robustness}

\author{
	Samuel Yeom\And
	Matt Fredrikson
	\affiliations
	Carnegie Mellon University
	\emails
	\{syeom, mfredrik\}@cs.cmu.edu
}

\newtheorem{definition}{Definition}
\newtheorem{theorem}{Theorem}
\newtheorem{lemma}[theorem]{Lemma}

\DeclareMathOperator*{\argmax}{argmax}

\makeatletter
\newcommand{\citeyearpar}[2][]{%
	\def\tmp{#1}%
	\ifx\tmp\@empty
		\citeauthor{#2}~[\citeyear{#2}]%
	\else
		\citeauthor{#2}~[\citeyear[#1]{#2}]%
	\fi
}
\makeatother

\newcommand{\X}{\mathcal{X}}
\newcommand{\Y}{\mathcal{Y}}
\newcommand{\M}{\mathcal{M}}
\newcommand{\R}{\mathbb{R}}
\newcommand{\E}{\mathbb{E}}
\newcommand{\x}{\bm{x}}
\renewcommand{\t}{\bm{t}}
\newcommand{\e}{\bm{e}}
\newcommand{\beps}{\bm{\epsilon}}
\newcommand{\bsigma}{\bm{\Sigma}}

\newcommand{\fstar}{f^\star}
\newcommand{\sgn}{\mathrm{sgn}}
\newcommand{\Var}{\mathrm{Var}}
\newcommand{\one}{\mathds{1}}
\newcommand{\binset}{\{0, 1\}}

\newcommand{\Mall}{\M_{\mathrm{all}}}
\newcommand{\ML}{\M_L}
\newcommand{\dX}{D_\X}
\newcommand{\dY}{D_\Y}
\newcommand{\ddY}{D_{\Delta(\Y)}}
\newcommand{\Lp}{L^p}
\newcommand{\smodel}{h_{f,g}} %smoothed model
\newcommand{\smodelstar}{h_{\fstar,g}}
\newcommand{\sampmodel}{h_{f,g}^{n}}
\newcommand{\diff}{\mathrm{diff}}

\newif\ifarxiv\arxivtrue

\begin{document}

\maketitle

%!TEX root=paper.tex

\begin{abstract}
We turn the definition of individual fairness on its head---rather than ascertaining the fairness of a model given a predetermined metric, we find a metric for a given model that satisfies individual fairness.
This can facilitate the discussion on the fairness of a model, addressing the issue that it may be difficult to specify a priori a suitable metric.
Our contributions are twofold:
First, we introduce the definition of a \emph{minimal} metric and characterize the behavior of models in terms of minimal metrics.
Second, for more complicated models, we apply the mechanism of randomized smoothing from adversarial robustness to make them individually fair under a given weighted $\Lp$ metric.
Our experiments show that adapting the minimal metrics of linear models to more complicated neural networks can lead to meaningful and interpretable fairness guarantees at little cost to utility.
\end{abstract}

\section{Introduction}
When machine learning models are deployed to make predictions about people, it is important that the model treats individuals fairly.
\emph{Individual fairness}~\cite{dwork2012fairness} captures the notion that similar people should be treated similarly by imposing a continuity requirement on models.
However, this raises the difficult societal question of how to define which people are ``similar''.

We start in \cref{sec:minimal} from the insight that it may be easier to determine whether a given similarity metric is reasonable than it is to construct one from scratch.
Thus, rather than imposing individual fairness with a predetermined similarity metric, we find a metric that corresponds to the behavior of a given model, which can then guide the discussion on whether the model is fair.
To facilitate this, we introduce the notion of a \emph{minimal} fairness metric, and show that in many cases there exists a unique metric that best characterizes the behavior of a given model for this purpose.

In \cref{sec:smoothing}, we deal with more complicated models, such as deep neural networks, whose minimal metrics are not easily computable.
We show that we can make \emph{any} model provably individually fair by post-processing it with \emph{randomized smoothing}~\cite{cohen2019certified} to impose a given weighted $\Lp$ metric.
As randomized smoothing was originally applied as a defense against adversarial examples, our result brings to light the connection between individual fairness and adversarial robustness.
However, our theorems are in a sense stronger because individual fairness is a uniform requirement that applies to all points in the input space, whereas the certified threshold of Cohen \textit{et al.}\ is a function of the input point.
Our Laplace and Gaussian smoothing mechanisms are versatile in that they can make a model provably individually fair under any given weighted $\Lp$ metric, and we show the minimality of this metric for the smoothed model to argue that we do not add more noise than is necessary.

Finally, our experiments combine the two main elements of our paper---we smooth neural networks to be individually fair under a metric that is proportional to the minimal metrics of linear models trained on the same datasets.
Our results on four real datasets show that the neural networks smoothed with Gaussian noise in particular are often approximately as accurate as the original models.
Moreover, we can achieve models with similar favorable individual fairness guarantees to those of linear models while still enjoying the increased predictive accuracy enabled by the neural network.

\section{Related Work}
\citeyearpar{dwork2012fairness} introduced the definition of individual fairness, which contrasts with group-based notions of fairness~\cite{hardt2016equality,zafar2017fairness-www} that require demographic groups to be treated similarly on average.
Motivated in part by group fairness, \citeyearpar{zemel2013learning} learn a representation of the data that excludes information about a protected attribute, such as race or gender, whose use is often legally prohibited.
This work has spurred more research on fair representations~\cite{calmon2017optimized,madras2018learning,tan2019learning}, and the resulting representations implicitly define a similarity metric.
However, unlike the weighted $\Lp$ metrics that we use, these metrics are harder for humans to interpret and are primarily designed to attain group fairness.

Others approximate individual fairness based on a limited number of oracle queries, which represent human judgments, about whether pair of individuals is similar.
\citeyearpar{gillen2018online} attempt to learn a similarity metric that is consistent with the human judgments in the setting of online linear contextual bandits.
In a more general setting, \citeyearpar{ilvento2019metric} derives an approximate metric using comparison queries that ask which of two individuals a given third individual is more similar to.
Finally, \citeyearpar{jung2019eliciting} apply constrained optimization directly without assuming that the human judgments are consistent with a metric.

By contrast, we post-process a model using randomized smoothing to provably ensure individual fairness.
\citeyearpar{cohen2019certified} previously analyzed randomized smoothing in the context of adversarial robustness.
In the context of fairness, most post-processing approaches~\cite{hardt2016equality,canetti2019soft} do not take individual fairness into account, and although \citeyearpar{lohia2019bias} consider individual fairness, they define two individuals to be similar if and only if they differ only in the pre-specified protected attribute.

\section{Background}
In this section, we present the definitions and notation that we will use throughout the paper.
\begin{definition}[Distance metric] \label{def:metric}
A nonnegative function $D: \X \times \X \to \R$ is a \emph{distance metric} in $\X$ if it satisfies the following three conditions: nonnegativity, symmetry, and triangle inequality.
\end{definition}
In common mathematical usage, \cref{def:metric} is a pseudometric, and metrics must also satisfy the condition that $D(x_1, x_2) = 0$ if and only if $x_1 = x_2$.
However, throughout this paper we will refer to pseudometrics as metrics, following the convention in the field of metric learning.

One commonly used family of metrics is the standard $\Lp$ metric, which is defined over $\R^d$.
In this paper, we consider a more general family of metrics that allows each coordinate to be weighted differently.
\begin{definition}[Weighted $\Lp$ metric] \label{def:lp}
The \emph{weighted $\Lp$ metric}, with $p \ge 1$ and weights $w_i \ge 0$, is a distance metric in $\R^d$ that is defined by the equation
\begin{equation} \label{eqn:lp}
D(\x_1, \x_2) = \sqrt[p]{\textstyle\sum_{i=1}^d w_i \cdot |x_{1i} - x_{2i}|^p},
\end{equation}
where $x_{1i}$ and $x_{2i}$ are the $i$-th coordinates of $\x_1$ and $\x_2$, respectively.
\end{definition}
We place the restriction that $p \ge 1$ because otherwise the function $D$ does not satisfy the triangle inequality.
When $w_i=1$ for all $i$, we have the \emph{standard $\Lp$ metric}.

Throughout this paper, we will use $\X$ and $\Y$ to denote a model's input and output spaces, respectively.
Moreover, we will assume a distance metric $\dY: \Y \times \Y \to \R$ that characterizes how close two points in the output space are.
\begin{definition}[Individual fairness~\cite{dwork2012fairness}] \label{def:indivfair}
A model $h: \X \to \Y$ is \emph{individually fair} under metric $\dX: \X \times \X \to \R$ if, for all $x_1, x_2 \in \X$,
\begin{equation} \label{eqn:indivfair}
\dY(h(x_1), h(x_2)) \le \dX(x_1, x_2).
\end{equation}
\end{definition}
Individual fairness captures the intuition that the model should not behave arbitrarily.
In particular, it formalizes the notion that similar individuals should be treated similarly, i.e., given two individuals $x_1, x_2 \in \X$, if the distance $\dX(x_1, x_2)$ between them is small, then the distance $\dY(h(x_1), h(x_2))$ between the outputs of the model on these individuals should also be small.

\section{Minimal Distance Metric} \label{sec:minimal}
One criticism of individual fairness is that it is difficult to apply in practice because it requires one to specify the metric $\dX$~\cite{chouldechova2018frontiers}.
The choice of a metric in $\X$ dictates which individuals should be considered similar, which is highly context-dependent and often controversial.
Thus, we take a slightly different approach---rather than specifying a metric $\dX$ and asking whether a model is individually fair under that metric, we find one metric under which the model is individually fair.
Then, we can reason about whether the metric is appropriate for the task at hand.

However, there could be multiple metrics for which a model is individually fair.
In fact, if $\dX(x_1, x_2) \ge \dX'(x_1, x_2)$ for all $x_1, x_2 \in \X$, then any model that is individually fair under $\dX'$ is also fair under $\dX$, as the metrics are simply upper bounds on the extent to which a model's outputs can vary.
On the other hand, our goal is to characterize the behavior of a model, for which we need a \emph{tight} upper bound.
This notion of tightness is captured by the minimality of a distance metric, defined in \cref{def:minimal}.

\begin{definition}[Minimal distance metric] \label{def:minimal}
Let $\M$ be a set of distance metrics in $\X$.
A metric $\dX \in \M$ is \emph{minimal} in $\M$ with respect to model $h: \X \to \Y$ if \emph{(1)} $h$ is individually fair under $\dX$, and \emph{(2)} there does not exist a different $\dX' \in \M$ such that $h$ is individually fair under $\dX'$ and $\dX(x_1, x_2) \ge \dX'(x_1, x_2)$ for all $x_1, x_2 \in \X$.
\end{definition}

To see how one may reason about the minimal distance metric, consider a hiring model with a binary output that informs whether a given applicant should be hired.
A natural $\dY$ in this setting is the 0-1 loss $\dY(y_1, y_2) = \one[y_1 \neq y_2]$.
Then, if the hiring model satisfies individual fairness under a metric $\dX$ such that $\dX(x_1, x_2) = 0$ whenever $x_1$ and $x_2$ differ only in race, we can reason that it does not directly use race to discriminate.

We now present \cref{thm:allmetrics}, which identifies the unique minimal metric among the set of all metrics.
\begin{theorem} \label{thm:allmetrics}
Let $h: \X \to \Y$ be a model, and let $\Mall$ be the set of all metrics that satisfy the conditions in \cref{def:metric}.
Then, the metric $\dX$, defined as
\begin{equation} \label{eqn:allmetrics}
\dX(x_1, x_2) = \dY(h(x_1), h(x_2))
\end{equation}
for all $x_1, x_2 \in \X$, is the unique minimal metric in $\Mall$ with respect to $h$.
\end{theorem}
\begin{proof}
We first prove that $\dX$ is a minimal metric, and later we will prove that no other metric is minimal.
Since we assume $\dY$ to be a metric, it easily follows that $\dX$ is also a metric under \cref{def:metric}.
Moreover, the equality in \cref{eqn:indivfair} always holds by our definition of $\dX$, so $h$ is individually fair under $\dX$.
Thus, it remains to show that there does not exist a different $\dX' \in \Mall$ such that $h$ is individually fair under $\dX'$ and $\dX(x_1, x_2) \ge \dX'(x_1, x_2)$ for all $x_1, x_2 \in \X$.

Suppose such $\dX'$ exists.
Since $\dX' \neq \dX$, there must exist some $x_1, x_2 \in \X$ such that $\dX(x_1, x_2) > \dX'(x_1, x_2)$. This, combined with \cref{eqn:allmetrics}, contradicts our assumption that $h$ is individually fair under $\dX'$.

Now we prove that $\dX$ is the unique minimal metric, arguing that $\dX'$ cannot be minimal if $\dX' \neq \dX$.
If there exist $x_1, x_2 \in \X$ such that $\dX(x_1, x_2) > \dX'(x_1, x_2)$, then $h$ is not individually fair under $\dX'$.
Thus, we must have $\dX'(x_1, x_2) \ge \dX(x_1, x_2)$ for all $x_1, x_2 \in \X$, but then $h$ is individually fair under $\dX$, so $\dX'$ cannot be minimal.
\end{proof}

\cref{thm:allmetrics} shows that the minimal metric $\dX$ in $\Mall$ is defined directly in terms of the model in question.
Ideally, we want the minimal metric to be simpler than the model so that it can help us interpret and reason about the fairness of the model.
Thus, in the rest of this paper we only consider weighted $\Lp$ metrics, which comprise a broad and interpretable family of metrics defined over $\R^d$.

With this set of metrics, we can no longer prove a theorem as general as \cref{thm:allmetrics}, so we now prove a result for linear regression models.
In this setting, we have $\Y = \R$, and the distance metric is simply the absolute value $\dY(y_1, y_2) = |y_1 - y_2|$.
\cref{thm:lplinear} identifies the weighted $\Lp$ metric that is uniquely minimal for a given linear regression model.

\begin{theorem} \label{thm:lplinear}
Let $h: \R^d \to \R$ be a linear regression model with coefficients $\beta_1, \ldots, \beta_d$, and let $\ML$ be the set of all weighted $\Lp$ metrics.
Then, the $L^1$ metric with weights $w_i = |\beta_i|$ is the unique minimal metric in $\ML$ with respect to $h$.
\end{theorem}
\begin{proof}
$\dX$ is clearly in $\ML$ by definition.
To see that $h$ is individually fair under $\dX$, note that for all $\x_1, \x_2 \in \R^d$
\begin{equation} \label{eqn:lplinear}
\begin{multlined}
\dY(h(\x_1), h(\x_2))
= \textstyle |\sum_{i=1}^d \beta_i (x_{1i} - x_{2i})| \\
\le \textstyle \sum_{i=1}^d |\beta_i (x_{1i} - x_{2i})|
= \dX(\x_1, \x_2).
\end{multlined}
\end{equation}

The rest of the proof closely mirrors the argument given in the proof of \cref{thm:allmetrics}, so we only mention how the proofs differ.
As in the proof of \cref{thm:allmetrics}, we assume that there exists $\x_1, \x_2 \in \R^d$ such that $\dX(\x_1, \x_2) > \dX'(\x_1, \x_2)$.
Our goal is to show that $\dX'$ is not individually fair, and for this proof we have the additional condition that $\dX' \in \ML$.
However, it is not necessarily true that $\dX'(\x_1, \x_2) < \dY(h(\x_1), h(\x_2))$, so we instead construct $\x'_2$ such that $\dX'(\x_1, \x'_2) < \dY(h(\x_1), h(\x'_2))$.

Let $x'_{2i} = x_{1i} - \sgn(\beta_i) |x_{1i} - x_{2i}|$.
With \cref{eqn:lp}, we can verify that $\dX(\x_1, \x_2) = \dX(\x_1, \x'_2)$ for any $\dX \in \ML$.
Moreover, $\beta_i (x_{1i} - x'_{2i}) \ge 0$ for all $i$, so the equality in \cref{eqn:lplinear} holds if we replace $\x_2$ by $\x'_2$.
Combining all of these relations, we arrive at the desired result:
\begin{multline*}
\dX'(\x_1, \x'_2) = \dX'(\x_1, \x_2) < \dX(\x_1, \x_2) \\
= \dX(\x_1, \x'_2) = \dY(h(\x_1), h(\x'_2)). \qedhere
\end{multline*}
\end{proof}

\section{Randomized Smoothing} \label{sec:smoothing}
For settings without a simple linear relation between the inputs and the outputs, neural networks often replace linear models.
However, neural networks are often susceptible to adversarial examples~\cite{szegedy2014intriguing,goodfellow2015explaining}, %neural networks are unlikely to be individually fair under distance metrics of reasonable size.
which are inputs to the model that are created by applying a small perturbation to an original input with the goal of causing a very large change in the model's output.
The frequent success of these attacks show that a small change in $\X$ can cause a large change in $\Y$, which is contrary to individual fairness.

Previously, \citeyearpar{cohen2019certified} introduced randomized smoothing, a post-processing method that ensures that the post-processed model is robust against perturbations of size, measured with the standard $L^2$ norm, up to a threshold that depends on the input point.
In this section, for any given metric, we apply a modified version of randomized smoothing and prove that the resulting model is individually fair under that metric.
We note that this result does not immediately follow from prior results---individual fairness imposes the same constraint on every point in the input space, whereas the certified threshold of Cohen \textit{et al.}\ is a function of the input point.

In the rest of this paper, we assume that $\Y$ is categorical, following the setting of \citeyearpar{cohen2019certified}.
In this section, we present and prove two methods for deriving an individually fair model from an arbitrary function $f: \R^d \to \Y$.
Like the models considered by \citeyearpar{dwork2012fairness}, our fair model $\smodel$ maps $\R^d$ to $\Delta(\Y)$, which is the set of probability distributions over $\Y$.
It is important to note that $\smodel$ is deterministic and that we treat its output simply as an array of probabilities.
To avoid confusion with the randomness that we introduce in \cref{sec:applications}, we will write $\smodel(\x)[y]$ to denote the probability $\Pr[\smodel(\x) = y]$.
\begin{definition}[Randomized smoothing] \label{def:smoothing}
Let $f: \R^d \to \Y$ be an arbitrary model, and let $g: \R^d \to \R$ be a probability distribution\footnote{We abuse notation and use $g$ to denote both the distribution and its probability density function.}.
Then, the \emph{smoothed model} $\smodel: \R^d \to \Delta(\Y)$ is defined by
\begin{equation} \label{eqn:smoothing}
\smodel(\x)[y] = \int_{\R^d} \one[f(\x+\t) = y] \cdot g(\t) \, d\t
\end{equation}
for all $y \in \Y$, and $g$ is called the \emph{smoothing distribution}.
\end{definition}

Intuitively, $f$ is the original model, and the value of the smoothed model $\smodel$ at $\x$ is found by querying $f$ on points around $\x$.
We choose the points around $\x$ according to the distribution $g$, and the output $\smodel(\x)$ of the smoothed model is a probability distribution of the values of $f$ at the queried points.
To reason about the individual fairness of $\smodel$, we use the total variation distance (\cref{eqn:tv}) to define the distance $\ddY$ between probability distributions.
\begin{equation} \label{eqn:tv}
\textstyle \ddY(Y_1, Y_2) = \frac{1}{2} \sum_{y \in \Y} |Y_1[y] - Y_2[y]|.
\end{equation}

\subsection{Laplace Smoothing Distribution} \label{sec:laplace}
One main difference between this setting and that in \cref{sec:minimal} is that we have a choice of the smoothing distribution $g$.
Thus, instead of simply finding a metric under which the model is individually fair, we adapt the smoothing distribution to a given metric.
\cref{thm:laplacefair} shows that, for any weighted $\Lp$ metric $\dX$, there exists a smoothing distribution $g$ that guarantees that $\smodel$ is individually fair under $\dX$ for all $f$.

\begin{theorem}[Laplace smoothing] \label{thm:laplacefair}
Let $\ML$ be the set of all weighted $\Lp$ metrics.
For any $\dX \in \ML$, let $g(\t) = \exp(-2\dX(\bm{0},\t)) / Z$, where $Z$ is the normalization factor $\int_{\R^d} \exp(-2\dX(\bm{0},\t)) \, d\t$.
Then, $\smodel$ is individually fair under $\dX$ for all $f$.
\end{theorem}
\begin{proof}
We will show that $\ddY(\smodel(\x), \smodel(\x+\beps)) \le \dX(\x, \x+\beps)$ for all $\x, \beps \in \R^d$.

First, since $\dX(\t, \t-\beps) = \dX(\bm{0}, \beps)$ for all weighted $\Lp$ metric $\dX$, we have
\begin{equation} \label{eqn:triangle}
\dX(\bm{0}, \t) - \dX(\bm{0}, \beps) \le \dX(\bm{0}, \t-\beps) \le \dX(\bm{0}, \t) + \dX(\bm{0}, \beps)
\end{equation}
by the triangle inequality.
We can apply the first inequality in \cref{eqn:triangle} to bound the probability $g(\t-\beps)$ in terms of $g(\t)$.
\begin{align*}
g(\t-\beps) &= \exp(-2\dX(\bm{0},\t-\beps)) / Z \\
&\le \exp(-2[\dX(\bm{0},\t) - \dX(\bm{0}, \beps)]) / Z \\
&= \exp(-2\dX(\bm{0},\t)) / Z \cdot \exp(2\dX(\bm{0}, \beps)) \\
&= g(\t) \cdot \exp(2\dX(\bm{0}, \beps))
\end{align*}
Then, for all $y \in \Y$ we have
\begin{equation} \label{eqn:problower}
\begin{aligned}
&\smodel(\x+\beps)[y] \\
&\textstyle = \int_{\R^d} \one[f(\x+\beps+\t) = y] \cdot g(\t) \, d\t \\
&\textstyle = \int_{\R^d} \one[f(\x+\t) = y] \cdot g(\t-\beps) \, d\t \\
&\textstyle \le \int_{\R^d} \one[f(\x+\t) = y] \cdot g(\t) \cdot \exp(2\dX(\bm{0}, \beps)) \, d\t \\
&= \smodel(\x)[y] \cdot \exp(2\dX(\bm{0}, \beps)).
\end{aligned}
\end{equation}
Similarly, we can apply the second inequality in \cref{eqn:triangle} to derive the upper bound
\begin{equation} \label{eqn:probupper}
\smodel(\x+\beps)[y] \ge \smodel(\x)[y] / \exp(2\dX(\bm{0}, \beps)).
\end{equation}

We can now apply a previous result by \citeyearpar[Theorem 6]{kairouz2016extremal} to determine the maximum distance between $\smodel(\x)$ and $\smodel(\x+\beps)$ that is attainable with the above constraints.
For brevity, let $c$ denote $\exp(2\dX(\bm{0}, \beps))$.
In the context of $\varepsilon$-local differential privacy, Kairouz \textit{et al.}\ showed that the maximum possible total variation distance is $(e^\varepsilon-1) / (e^\varepsilon+1)$.
Replacing $e^\varepsilon$ with $c$, we see that the distance between $\smodel(\x)$ and $\smodel(\x+\beps)$ is at most $(c-1) / (c+1)$.

Finally, it remains to be proven that this quantity is not more than $\dX(\x, \x+\beps)$, which is equivalent to $\dX(\bm{0}, \beps)$ for weighted $\Lp$ metrics.
Since this distance can be written as $(\ln c) / 2$, it suffices to show that $(c-1) / (c+1) \le (\ln c) / 2$ for all $c \ge 1$.
This inequality follows from the fact that equality holds at $c=1$ and that the derivative of the right-hand side is never less than that of the left-hand side for $c \ge 1$.
\end{proof}

Although \cref{thm:laplacefair} identifies a smoothing distribution that ensures the individual \emph{fairness} of the resulting model under $\dX$, we also want the smoothed model to retain the \emph{utility} of the original model $f$.
In the extreme case where $g$ is the uniform distribution over $\R^d$, the resulting model will be a constant function and therefore satisfy individual fairness under any metric, but it will not be very useful for classification tasks.
More generally, smoothed models that are individually fair under smaller distance metrics tend to not preserve as much locally relevant information about $f$.
Thus, we argue that a smoothing distribution does not unnecessarily lower the model's utility by showing that $\dX$ is \emph{minimal}.
The definition of minimality that we use here differs from \cref{def:minimal} in that the smoothed model must be individually fair for all $f$.

\begin{definition}[Minimal distance metric, smoothing] \label{def:minimalsmoothing}
Let $\M \subseteq \ML$ be a set of distance metrics in $\R^d$.
A metric $\dX \in \M$ is \emph{minimal} in $\M$ with respect to a smoothing distribution $g$ if \emph{(1)} $\smodel$ is individually fair under $\dX$ for all $f$, and \emph{(2)} there does not exist a different $\dX' \in \M$ such that $\smodel$ is individually fair under $\dX'$ for all $f$ and $\dX(\x_1, \x_2) \ge \dX'(\x_1, \x_2)$ for all $\x_1, \x_2 \in \R^d$.
\end{definition}

For general weighted $\Lp$ metrics, the inequalities in \cref{eqn:triangle} are strict for most $\t$, so the bounds in \cref{eqn:problower,eqn:probupper} are not tight, and $\dX$ is not guaranteed to be minimal.
On the other hand, if $\dX$ is a weighted $L^1$ metric, \cref{thm:laplaceminimal} shows that it is minimal with respect to its Laplace smoothing distribution.

\begin{theorem} \label{thm:laplaceminimal}
Let $\M_1$ be the set of all weighted $L^1$ metrics.
For any $\dX \in \M_1$, let $g(\t) = \exp(-2\dX(\bm{0},\t)) / Z$, where $Z$ is the normalization factor $\int_{\R^d} \exp(-2\dX(\bm{0},\t)) \, d\t$.
Then, $\dX$ is uniquely minimal in $\M_1$ with respect to $g$.
\end{theorem}
\begin{proof}
We have already proven in \cref{thm:laplacefair} that $\smodel$ is individually fair under $\dX$ for all $f$.
It remains to show that there does not exist a different $\dX' \in \M_1$ such that $\smodel$ is individually fair under $\dX'$ for all $f$ and $\dX(\x_1, \x_2) \ge \dX'(\x_1, \x_2)$ for all $\x_1, \x_2 \in \R^d$.

Let $w_1, \ldots, w_d$ and $w'_1, \ldots, w'_d$ be the weights of $\dX$ and $\dX'$, respectively.
If $\dX(\x_1, \x_2) \ge \dX'(\x_1, \x_2)$ for all $\x_1, \x_2 \in \R^d$, we must have $w_i \ge w'_i$ for all $i$.
Moreover, since $\dX \neq \dX'$, there exists $i$ such that $w_i > w'_i$.
We now construct $f$ such that $\smodel$ is not individually fair under $\dX'$.

Let $f: \R^d \to \binset$ be a function such that $f(\x) = \one[x_i \ge 0]$.
We will show that there exists $\epsilon > 0$ such that $\ddY(\smodel(\bm{0}), \smodel(\epsilon \e_i)) > \dX'(\bm{0}, \epsilon \e_i)$, where $\e_i$ is the basis vector that is one in the $i$-th coordinate and zero in all others.
Applying \cref{eqn:smoothing} and simplifying, we get $\smodel(\bm{0})[0] = 1/2$ and $\smodel(\epsilon \e_i)[0] = \exp(-2w_i\epsilon)/2$.
Therefore, the distance $\ddY$ is $(1 - \exp(-2w_i\epsilon))/2$.
Moreover, we have $\dX' = w'_i\epsilon$.
The ratio $\ddY/\dX'$ approaches $w_i/w'_i > 1$ as $\epsilon \to 0$, so when $\epsilon$ is sufficiently small, we have $\ddY > \dX'$.

Uniqueness follows from the argument given in the last paragraph of the proof of \cref{thm:allmetrics}.
\end{proof}

\subsection{Gaussian Smoothing Distribution} \label{sec:gaussian}
As we show in \cref{sec:experiments}, in practice Laplace smoothing distributions do not preserve well the utility of $f$ due to their relatively high densities at the tails.
Thus, we present Gaussian smoothing as an alternative, which \cref{thm:gaussianfair} shows is individually fair under any weighted $L^2$ metric.
Since $D_2(\x_1, \x_2) \le d^{\max(0, 1/2 - 1/p)} D_p(\x_1, \x_2)$ for any weighted $L^2$ and $\Lp$ metrics $D_2$ and $D_p$ with the same weights, we can then scale the weights accordingly to make $\smodel$ fair under any given $\Lp$ metric.
For simplicity, we only consider the setting of binary classification, i.e., $\Y = \binset$.

\begin{theorem}[Gaussian smoothing] \label{thm:gaussianfair}
Let $\dX$ be a weighted $L^2$ metric with weights $w_1, \ldots, w_d$, and let $\bsigma$ be a diagonal matrix with $\Sigma_{ii} = (2\pi w_i)^{-1}$.
If $g$ is Gaussian with mean $\bm{0}$ and variance $\bsigma$, $\smodel$ is individually fair under $\dX$ for all $f: \R^d \to \binset$.
\end{theorem}

To prove this theorem, we will apply the Neyman--Pearson lemma~\cite{neyman1933problem}, as formulated by \citeyearpar[Lemma 3]{cohen2019certified}.

\begin{lemma}[Neyman--Pearson] \label{lem:neyman}
Let $X_1$ and $X_2$ be random variables in $\R^d$ with densities $\mu_{X_1}$ and $\mu_{X_2}$, and let $f, \fstar: \R^d \to \binset$ such that $\fstar(\t) = 1$ if and only if $\mu_{X_2}(\t) / \mu_{X_1}(\t) \ge k$ for some threshold $k > 0$.
Then,
\[
\begin{multlined}
\Pr[f(X_1) = 1] = \Pr[\fstar(X_1) = 1] \\
\text{implies } \Pr[f(X_2) = 1] \le \Pr[\fstar(X_2) = 1].
\end{multlined}
\]
\end{lemma}

\begin{proof}[Proof of \cref{thm:gaussianfair}]
We proceed by showing that $\ddY(\smodel(\x), \smodel(\x+\beps)) \le \dX(\x, \x+\beps)$ for all $\x, \beps \in \R^d$ and $f: \R^d \to \binset$.
For any given $f$, we will first find $\fstar$ such that
\begin{equation} \label{eqn:fstar}
\begin{multlined}
\ddY(\smodel(\x), \smodel(\x+\beps)) \\
\le \ddY(\smodelstar(\x), \smodelstar(\x+\beps)).
\end{multlined}
\end{equation}
We will then show that $\smodelstar$ is individually fair under $\dX$, which together with \cref{eqn:fstar} implies that $\smodel$ is also individually fair under $\dX$.

Fix $\x, \beps \in \R^d$, and assume without loss of generality that $\smodel(\x)[1] \le \smodel(\x+\beps)[1]$.
Then, we have
\begin{equation} \label{eqn:binarydist}
\ddY(\smodel(\x), \smodel(\x+\beps)) = \smodel(\x+\beps)[1] - \smodel(\x)[1].
\end{equation}
We apply \cref{lem:neyman} by choosing $X_1$ and $X_2$ such that $g(\t) = \mu_{X_1}(\x+\t) = \mu_{X_2}(\x+\beps+\t)$.
By \cref{eqn:smoothing}, we have $\Pr[f(X_1) = 1] = \smodel(\x)[1]$ and $\Pr[f(X_2) = 1] = \smodel(\x+\beps)[1]$, and similar relations hold between $\fstar$ and $\smodelstar$.
Therefore, if there exists $\fstar$ that satisfies the condition in \cref{lem:neyman} such that $\smodel(\x)[1] = \smodelstar(\x)[1]$, then $\smodel(\x+\beps)[1] \le \smodelstar(\x+\beps)[1]$.
Combining these two (in)equalities, we get
\[ \smodel(\x+\beps)[1] - \smodel(\x)[1] \le \smodelstar(\x+\beps)[1] - \smodelstar(\x)[1], \] and \cref{eqn:fstar} follows from \cref{eqn:binarydist} and its $\smodelstar$ counterpart.

We now show that it is possible to find $\fstar$ such that $\smodelstar(\x)[1] = \smodel(\x)[1]$.
By construction, we have that $\fstar(\t) = 1$ if and only if
\[ \frac{\mu_{X_2}(\t)}{\mu_{X_1}(\t)} = \frac{g(\t-\x-\beps)}{g(\t-\x)} \ge k \]
for some $k > 0$.
Substituting in the Gaussian density function and solving for $\t$, we see that this inequality holds whenever $\beps^T \bsigma^{-1} \t \ge \kappa$, where $\kappa$ is a constant with respect to $\t$.
When evaluating $\smodelstar(\x)$ as per \cref{eqn:smoothing}, $\t$ is distributed normally, and therefore $\beps^T \bsigma^{-1} \t$ is also (univariate) Gaussian.
Thus, with the appropriate value of $\kappa$ we can obtain the desired $\fstar$.

Finally, it remains to show that $\smodelstar$ is individually fair under $\dX$.
Let $\tau = \beps^T \bsigma^{-1} \t$, and let $\gamma$ be the density function of $\tau$.
With some computation, we see that $\fstar(\t)$ and $\fstar(\t+\beps)$ differ if and only if $\kappa \le \tau < \kappa + \beps^T \bsigma^{-1} \beps$.
Moreover, since $\t$ has variance $\bsigma$, the variance of $\tau = \beps^T \bsigma^{-1} \t$ is $\beps^T \bsigma^{-1} \Var(\t) (\beps^T \bsigma^{-1})^T = \beps^T \bsigma^{-1} \beps$, and thus the maximum value of $\gamma$ is $(2\pi \beps^T \bsigma^{-1} \beps)^{-1/2}$.
We apply these two facts to arrive at the desired result:
\begin{align*}
&\ddY(\smodelstar(\x), \smodelstar(\x+\beps)) \\
&= \smodelstar(\x+\beps)[1] - \smodelstar(\x)[1] \\
&= \textstyle \int_{\R^d} (\fstar(\x+\beps+\t) - \fstar(\x+\t)) \cdot g(\t) \, d\t \\
&= \textstyle \int_{\R^d} (\fstar(\t+\beps) - \fstar(\t)) \cdot g(\t-\x) \, d\t \\
&= \textstyle \int_\kappa^{\kappa + \beps^T \bsigma^{-1} \beps} \gamma(\tau - \beps^T \bsigma^{-1} \x) \, d\tau \\
&\le \beps^T \bsigma^{-1} \beps \cdot (2\pi \beps^T \bsigma^{-1} \beps)^{-1/2} \\
&= \textstyle \sqrt{\sum_{i=1}^d \epsilon_i^2 / (2\pi \Sigma_{ii})} \\
&= \textstyle \sqrt{\sum_{i=1}^d \epsilon_i^2 \cdot w_i} = \dX(\x, \x+\beps). \qedhere
\end{align*}
\end{proof}

We end this section with \cref{thm:gaussianminimal}, which states that an $L^2$ metric $\dX$ is minimal with respect to its Gaussian smoothing distribution.
We omit the proof since it is very similar to that of \cref{thm:laplaceminimal}.

\begin{theorem} \label{thm:gaussianminimal}
Let $\M_2$ be the set of all weighted $L^2$ metrics.
For any $\dX \in \M_2$, let $w_1, \ldots, w_d$ be the weights, and let $\bsigma$ be a diagonal matrix with $\sigma_{ii} = (2\pi w_i)^{-1}$.
If $g$ is a Gaussian with mean $\bm{0}$ and variance $\bsigma$, $\dX$ is uniquely minimal in $\M_2$ with respect to $g$.
\end{theorem}

\section{Practical Implementation} \label{sec:applications}
In practice, it is infeasible to compute $\smodel(\x)$ because of the integral in \cref{eqn:smoothing}.
Therefore, to apply randomized smoothing in practice, we approximate the integral
\ifarxiv
with \cref{alg:smoothing}, i.e.,
\fi
by sampling $n$ points independently from the smoothing distribution $g$, evaluating the model with this noise added to $\x$, and returning the observed probability of predicting each class on the sampled points.
However, the resulting model may not be individually fair due to the finite sample size.
Thus, we define and prove $(\epsilon, \delta)$-individual fairness, which requires that the model be close to individually fair with high probability.

\ifarxiv
\begin{algorithm}[t] %smoothing algorithm
\begin{algorithmic}
\REQUIRE Model $f: \R^d \to \Y$, point $\x \in \R^d$, parameters $p \in \{1, 2, \infty\}$, $\bm{w} \in \R_+^d$ for the weighted $\Lp$ metric $\dX$
\ENSURE Distribution in $\Delta(\Y)$ that approximates $\smodel(\x)$

\vspace{1ex}

\STATE //Initialize the output probability distribution.
\FOR{$y \in \Y$}
	\STATE $\mathrm{prob}[y] \gets 0$
\ENDFOR

\STATE //Evaluate $f$ at $n$ randomly sampled points around $\x$.
\FOR{$j = 1, \ldots, n$}
	\STATE $\t_j \gets \mathrm{sampleNoise}(d, p, \bm{w})$
	\STATE $y_j \gets f(\x + \t_j)$
	\STATE $\mathrm{prob}[y_j] \gets \mathrm{prob}[y_j] + 1/n$
\ENDFOR

\RETURN $\mathrm{prob}$
\end{algorithmic}
\caption{Randomized smoothing by sampling} \label{alg:smoothing}
\end{algorithm}
\fi

\begin{definition}[$(\epsilon, \delta)$-individual fairness] \label{def:indivfairapprox}
A randomized model $h: \X \to \Delta(\Y)$ is \emph{$(\epsilon, \delta)$-individually fair} under metric $\dX: \X \times \X \to \R$ if, for all $x_1, x_2 \in \X$,
\begin{equation} \label{eqn:indivfairapprox}
\ddY(h(x_1), h(x_2)) \le \dX(x_1, x_2) + \epsilon
\end{equation}
with probability at least $1-\delta$.
The probability is taken over the randomness of $h$.
\end{definition}

\begin{theorem} \label{thm:smoothingapprox}
Let $\sampmodel$ be a model that approximates $\smodel$ with $n$ samples.
If $|\Y| = m$ and $\smodel$ is fair under $\dX$, then $\sampmodel$ is $(\epsilon, 2me^{-4n\epsilon^2/m^2})$-individually fair under $\dX$.
\end{theorem}
\begin{proof}
Consider any two points $\x_1, \x_2 \in \R^d$.
Since $\smodel$ is individually fair under $\dX$, we have
\begin{multline*}
\textstyle \frac{1}{2} \sum_{y \in \Y} |\smodel(\x_1)[y] - \smodel(\x_2)[y]| \\
= \ddY(\smodel(\x_1), \smodel(\x_2))
\le \dX(\x_1, \x_2).
\end{multline*}
We will show that $|\diff[y]| > 2\epsilon/m$ with probability less than $2e^{-4n\epsilon^2/m^2}$, where $\diff[y] = (\sampmodel(\x_1)[y] - \sampmodel(\x_2)[y]) - (\smodel(\x_1)[y] - \smodel(\x_2)[y])$.
Then, by union bound, with probability at least $1 - 2me^{-4n\epsilon^2/m^2}$ we will have $\diff[y] > 2\epsilon/m$ for all $y \in \Y$, which leads to our desired result.
\begin{align*}
&\ddY(\sampmodel(\x_1), \sampmodel(\x_2)) \\
% &\textstyle = \frac{1}{2} \sum_{y \in \Y} |\sampmodel(\x_1)[y] - \sampmodel(\x_2)[y]| \\
&\textstyle \le \frac{1}{2} \sum_{y \in \Y} |\smodel(\x_1)[y] - \smodel(\x_2)[y]| + \frac{1}{2} \sum_{y \in \Y} |\diff[y]| \\
&\le \dX(x_1, x_2) + \epsilon
\end{align*}

Fix $y \in \Y$, and let $X_{ij} = \one[f(\x_i + \t_{ij}) = y]$, where $\t_{ij}$ is the $j$-th sample drawn from the smoothing distribution $g$ while evaluating $\sampmodel(\x_i)$.
Then, $\sampmodel(\x_i)[y] = \frac{1}{n} \sum_{j=1}^n X_{ij}$ and $\smodel(\x_i)[y] = \frac{1}{n} \E[\sum_{j=1}^n X_{ij}]$, so $\diff = \frac{1}{n} \sum_{j=1}^n (X_{1j} - X_{2j} - \E[X_{1j} - X_{2j}])$.
The theorem follows from Hoeffding's inequality.
\begin{align*}
&\Pr[|\diff[y]| > 2\epsilon/m] \\
&= \textstyle \Pr[|\frac{1}{2n} \sum_{j=1}^n (X_{1j} - X_{2j} - \E[X_{1j} - X_{2j}])| > \epsilon/m] \\
&< 2e^{-4n\epsilon^2/m^2} \qedhere
\end{align*}
\end{proof}

\ifarxiv
%!TEX root=arxiv.tex

\subsection{Noise Sampling}
Implementations of Gaussian noise sampling are commonly included in data analysis libraries.
For Laplace noise sampling, we apply \cref{alg:sampling}, which describes how to sample a point $\t$ from the Laplace smoothing distribution when $p \in \{1, 2, \infty\}$.
Without loss of generality, we assume that $\dX$ is a standard $\Lp$ metric since we can simply rescale each coordinate by its weight $w_i$.
Recall that $g(\t) \propto \exp(-2\dX(\bm{0}, \t))$.
When $p = 1$, this quantity becomes $\exp(-2\sum_{i=1}^d w_i \cdot |t_i|) = \prod_{i=1}^d \exp(-2w_i \cdot |t_i|)$, so each coordinate can be sampled from the Laplace distribution independently of the others.
For other values of $p$, the coordinates are not independent, so we instead sample the distance $r = \dX(\bm{0}, \t) = \|\t\|_p$ and then pick a point $\t$ uniformly at random on the sphere ($p = 2$) or hypercube ($p = \infty$) of radius $r$.

To sample $r$, we note that the set $\{\t \mid \|\t\|_p = r\}$ has surface area proportional to $r^{d-1}$.
Hence, the probability of drawing a point in this set from the distribution $g$ is proportional to $r^{d-1} e^{-2r} \, dr$, and the cumulative distribution function of $r$ is $P(d, 2r)$, where $P$ is the regularized lower incomplete gamma function.
Finally, computing the inverse of this function allows us to sample $r$ through inverse transform sampling.

\begin{algorithm}[t] %sampling algorithm
\begin{algorithmic}
\REQUIRE Positive integer $d$, parameters $p \in \{1, 2, \infty\}$, $\bm{w} \in \R_+^d$ for the weighted $\Lp$ metric $\dX$
\ENSURE Noise $\t \in \R^d$ drawn randomly from distribution $g$ such that $g(\t) \propto \exp(-2\dX(\bm{0}, \t))$

\vspace{1ex}

\STATE //Sample noise under the standard $\Lp$ metric.
\IF{$p = 1$}
	\STATE //Each coordinate is independent when $p = 1$.
	\FOR{$i = 1, \ldots, d$}
		\STATE $t_i \sim \mathrm{Laplace}(0, 0.5)$
	\ENDFOR
\ELSE
	\STATE //Sample a random point on the unit $\Lp$-sphere.
	\FOR {$i = 1, \ldots, d$}
		\IF {$p = 2$}
			\STATE $t_i \sim \mathrm{Gaussian}(0, 1)$
		\ELSIF{$p = \infty$}
			\STATE $t_i \sim \mathrm{Uniform}(-1, 1)$
		\ENDIF
	\ENDFOR
	\STATE $\t \gets \t / \|\t\|_p$
	
	\STATE //Use inverse transform sampling for radius $r = \|\t\|_p$.
	\STATE $u \sim \mathrm{Uniform}(0, 1)$
	\STATE $r \gets P^{-1}(d, u) / 2$ \COMMENT{$P$ is reg.\ lower inc.\ gamma func.}
	\STATE $\t \gets r \cdot \t$
\ENDIF

\STATE //Adjust the noise, taking the weights into account.
\FOR {$i = 1, \ldots, d$}
	\STATE $t_i \gets t_i / w_i$
\ENDFOR
\RETURN $(t_1, \ldots, t_d)$
\end{algorithmic}
\caption{Laplace noise sampling} \label{alg:sampling}
\end{algorithm}
\else
In the extended version of this paper~\cite{yeom2020individual}, we provide pseudocode of an implementation of randomized smoothing.
\fi

\section{Experiments} \label{sec:experiments}
\cref{thm:laplacefair,thm:gaussianfair,thm:smoothingapprox} show that smoothed models created using
\ifarxiv
\cref{alg:smoothing}
\else
randomized smoothing
\fi
are individually fair, but we have no similar results about their utility except heuristic arguments from minimality.
In this section we measure the utility of smoothed models $\smodel$ on four real-world datasets (detailed below), using the smoothing distributions described in \cref{thm:laplacefair,thm:gaussianfair}.

The weights of $\dX$ were chosen to be proportional to those of a logistic regression model trained on the same dataset, so the features that receive little weight in the linear model, which are thus less likely to be predictive, have little effect on the output of the smoothed model.
The linear model weights were multiplied by a constant between 0.5 and 5 (depending on the dataset) to make the mean weight of $\dX$ equal 1.
For each dataset, we trained a neural network $f$ with two dense hidden layers of 128 ReLU neurons each, as well as a logistic regression model to use for deriving the targeted metric.
When training neural networks, we augmented training data with noise drawn from the smoothing distribution, as prior work~\cite{cohen2019certified} shows that this improves the utility of the smoothed model.
For smoothing, we sampled $n = 10^5$ points, which by \cref{thm:smoothingapprox} corresponds to a guarantee of $\delta = 1.8 \times 10^{-4}$ at $\epsilon = 10^{-2}$.

\begin{figure}
\resizebox{\columnwidth}{!}{\includegraphics{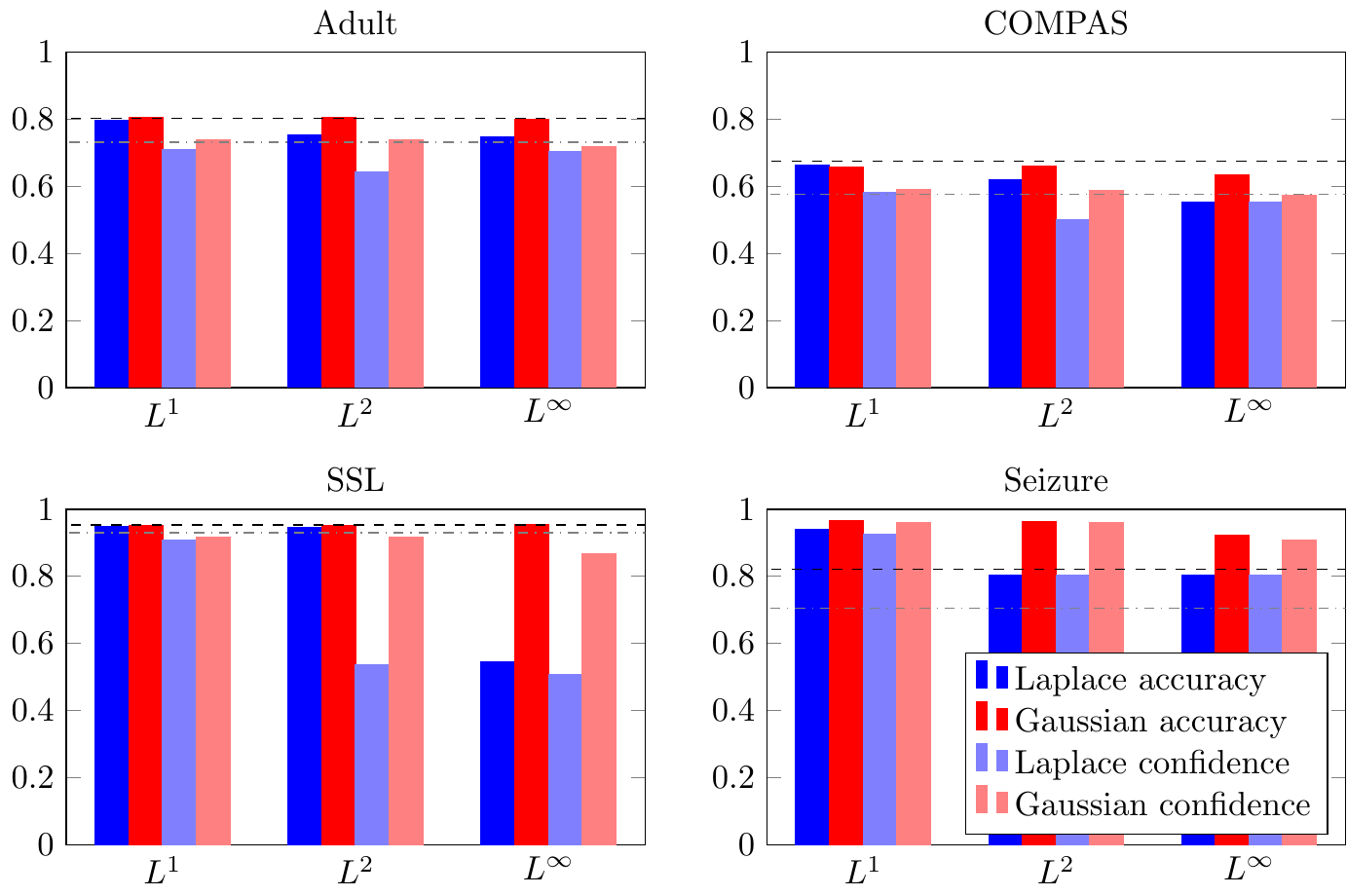}}
\caption{Utility of smoothed models derived from the four datasets described in \cref{sec:experiments}.
The black dashed line indicates the accuracy of the logistic regression model, and the dash-dotted line its average probit confidence.
Because smoothed models output probabilities, accuracy and mean confidence are both reasonable measures of their utility, but only mean confidence preserves the individual fairness of the smoothed model.
}
\label{fig:plot}
\end{figure}

\paragraph{Adult.}
Our model uses the five numerical features from the UCI Adult dataset~\cite{uci} to predict whether a person earns more than \$50,000 per year.

\paragraph{COMPAS.}
We use the dataset compiled by ProPublica~\cite{angwin2016machine} to analyze the COMPAS recidivism prediction model~\cite{compas}.
Our model uses eight features (15 when one-hot encoded) to predict whether a person will recidivate within the next two years.

\paragraph{SSL.}
The Strategic Subject List dataset~\cite{ssl} contains scores given by Chicago Police Department's model to rate a person's risk of being involved in a shooting incident, either as a perpetrator or a victim.
Our model uses the same eight numerical features used by Chicago's model to predict a person's SSL risk score.

\paragraph{Seizure.}
In the UCI Epileptic Seizure dataset~\cite{uci,andrzejak2001indications}, every row consists of 178 readings from an EEG taken over a second.
Our model predicts whether a person is experiencing a seizure during that second.

\subsection{Results}
We applied both Laplace and Gaussian smoothing to create models that are individually fair under the weighted $L^1$, $L^2$, or $L^\infty$ metrics, with weights derived from those of the corresponding logistic regression model as previously described.
Because the outputs of the smoothed models are probabilities, we measured their utilities both in terms of standard accuracy and the mean probit confidence assigned to the correct class, $\E_{(\x, y)}[\smodel(\x)[y]]$.
Although accuracy is a more common measure of utility, its use of the thresholding operator $\argmax$ is incompatible with the individual fairness of $\smodel$.

The results in \cref{fig:plot} show that Gaussian-smoothed models \emph{approximately match or exceed the performance of the logistic model while achieving similar individual fairness guarantees.}
We note that for all datasets except for Seizure, the accuracy of the unsmoothed neural network was within 1.5\% of the logistic regression model; on Seizure, the neural network achieved 96.6\% whereas the logistic model gave 82.2\% accuracy.
The Gaussian-smoothed Seizure model came very close ($\le 0.2\%$) to the unsmoothed neural network for $L^1$ and $L^2$ metrics, far exceeding the performance of the linear model.

We conclude by noting that, although $L^1$ metrics are minimal with respect to Laplace smoothing, Gaussian smooothing outperforms on these metrics in practice.
Laplace distributions have higher densities at the tails, resulting in more queries that are very dissimilar to the input point $\x$.
Thus, in practice it can be preferable to use Gaussian smoothing for every $\Lp$ metric, adjusting the weights as shown in \cref{sec:gaussian} to account for the value of $p$.

\section*{Acknowledgments}
This material is based upon work supported by Bosch Corporation, an NVIDIA GPU grant, and the National Science Foundation under Grant No.\ CNS-1704845.
The authors would like to thank Shayak Sen for his helpful feedback.

\bibliographystyle{named}
\bibliography{biblio}
\end{document}